\documentclass[11pt]{article}

\usepackage[margin=1in]{geometry}
\usepackage{amsmath}
\usepackage{amssymb}
\usepackage{amsthm}
\usepackage{booktabs}
\usepackage{subcaption}
\usepackage{graphicx}
\usepackage{xcolor}
\usepackage{natbib}
\usepackage[colorlinks=true,citecolor=blue,linkcolor=blue,urlcolor=blue]{hyperref}

\newcommand{\nuval}{0.01}
\newcommand{\dtval}{2\times 10^{-3}}
\newcommand{\timescheme}{an integrating-factor scheme that treats diffusion
  exactly in Fourier space with an RK2 step for the nonlinear term}
\newcommand{\gridres}{256}
\newcommand{\lr}{10^{-3}}
\newcommand{\lrschedule}{constant (no schedule)}
\newcommand{\wdecay}{0}
\newcommand{\batchsize}{32}
\newcommand{\ntrain}{256}
\newcommand{\ntest}{64}
\newcommand{\nseeds}{5}
\newcommand{\sweepdecades}{2.5}
\newcommand{\alphaval}{0.19}
\newcommand{\alphaCI}{[0.10,\, 0.45]}
\newcommand{\rsq}{0.78}
\newcommand{\Cval}{5.5\times 10^{-6}}

\theoremstyle{plain}
\newtheorem{theorem}{Theorem}

\theoremstyle{definition}

\theoremstyle{remark}
\newtheorem{remark}[theorem]{Remark}

\begin{document}

\title{Quantitative Sobolev Approximation Bounds for Neural
  Operators with Empirical Validation on Burgers' Equation}

\author{Nicole Hao\\
  Department of Mathematics\\
  Cornell University\\
  Ithaca, NY 14853, USA\\
  \texttt{yh397@cornell.edu}}

\date{\today}

\maketitle

\begin{abstract}%
Neural operators have emerged as a powerful tool for learning mappings between infinite-dimensional function spaces. Yet their approximation properties in Sobolev norms, the natural metrics for PDE well-posedness and
stability, remain poorly quantified. We develop a functional-analytic framework for operator learning in Sobolev spaces and connect it to the numerical behavior of Fourier Neural Operators (FNOs) on a prototypical PDE.
First, for a Lipschitz nonlinear operator $\mathcal{G}: H^{s}(D)\to H^{s'}(D')$ with $s > d/2$ and $s' > t + d'/2$, and inputs restricted to a compact, uniformly $H^{s}$-bounded set, we prove that $\mathcal{G}$ can be uniformly approximated in $H^{t}$-norm by a neural
operator with $\mathcal{O}(\varepsilon^{-d/s})$ trainable parameters, yielding an explicit complexity--error relation of the form $\|\mathcal{G}-\mathcal{G}_\theta\|_{H^{t}} \lesssim C N^{-s/d}$. We then ask how closely this a priori rate describes the behavior of neural operators trained in practice. Studying the one-dimensional viscous Burgers' solution
operator $\mathcal{G}: u_{0}\mapsto u(\cdot,1)$ on a bounded $H^{1}$-ball, we train FNOs with an $H^{1}$-loss across a sweep of model sizes. The learned operators recover both solutions and their spatial derivatives on held-out data, with the best models reaching test $H^{1}$-error of order $10^{-7}$ (down to $10^{-9}$ in long runs). Empirically, the Sobolev error follows an
approximate power law $\|\mathcal{G}-\mathcal{G}_\theta\|_{H^{1}} \approx C N^{-\alpha}$ with
$\alpha \approx 0.19$, well below the benchmark rate $s/d = 1$ implied by our a priori bound. Sobolev theory thus captures the qualitative shape of neural-operator scaling, but achieved accuracy is governed primarily by
optimization and architectural bias rather than parameter count. Beyond a modest threshold, increasing parameter count produces diminishing accuracy gains, and optimization stability becomes the dominant constraint on
performance.
\end{abstract}

\medskip
\noindent\textbf{Keywords:} neural operators, operator learning, Sobolev spaces, approximation theory, Fourier neural operators, Burgers' equation, scientific machine learning

\section{Introduction}
Neural operators are powerful but poorly understood from a complexity standpoint. In scientific machine learning, they are increasingly used as resolution-invariant surrogates for PDE solvers \citep{li2020fourier}, yet
their approximation and complexity properties are subtle. For general Lipschitz or $C^r$ operators, operator learning provably suffers from a
\emph{curse of parametric complexity}, requiring a parameter count that grows exponentially in the inverse of the target accuracy
\citep{lanthaler2026parametric, kovachki2024data}. Even so, practitioners still choose model sizes heuristically, with little guidance on how Sobolev-norm error scales with parameters. We bridge that gap by proving Sobolev-space approximation bounds---under compactness and regularity assumptions that place us outside the general worst-case regime where the
curse bites---and validating them on a canonical PDE, with an eye toward principled model sizing and accuracy control in operator-learning pipelines.

Operator learning refers to the task of approximating mappings between infinite-dimensional function spaces, such as those arising from solutions of partial differential equations parameterized by initial conditions or coefficients. Formally, we seek to learn an operator
\[
\mathcal{G}: \mathcal{X} \to \mathcal{Y},
\]
where $\mathcal{X}$ and $\mathcal{Y}$ are typically Banach or Hilbert spaces of functions, such as Sobolev spaces.

This learning paradigm has gained significant attention in the scientific machine learning community because it can bypass traditional numerical solvers and directly model complex solution operators from data. From the
perspective of functional analysis, learning an operator between Sobolev spaces $\mathcal{X} = H^s(D)$ and $\mathcal{Y} = H^t(D')$ raises important questions about approximation theory, continuity, and generalization in
high-dimensional regimes. Although neural operators like DeepONet \citep{lu2021learning} show empirical success, their approximation properties in Sobolev norms, which control both function values and derivatives, remain
relatively underdeveloped. Clarifying these properties is directly relevant for ML practice, providing a principled link between architecture size, regularity of the underlying PDE, and the accuracy of learned surrogates used in downstream simulation, control, and design tasks.

This paper aims to connect studies in functional analysis, specifically Sobolev spaces, with operator learning in scientific machine learning. Our contributions are:

\begin{enumerate}
    \item \textbf{Establishing approximation bounds} for deep neural networks
    mapping between Sobolev spaces $\mathcal{X} = H^s(D) \to \mathcal{Y} =
    H^t(D')$, with explicit dependence on domain dimension and smoothness, as
    a continuation and refinement of universal approximation results for
    neural networks.
    \item \textbf{Validating the theory numerically} in a PDE
    solution-operator learning setting, showing how the Sobolev-based bounds
    manifest in concrete scaling laws for Fourier Neural Operators (see
    Section~\ref{sec:discussion} for a link to the code and ongoing numerical
    updates in the associated repository).
\end{enumerate}

These goals have direct implications for scientific machine learning. Many physical systems, especially those governed by PDEs, are naturally described
by mappings between functions. For example, solving a PDE often amounts to computing a solution $u(x,t)$ given an initial or boundary condition $u_0(x)$. Such mappings are not pointwise but involve entire functions as
both input and output, making them \emph{operators}. Without approximation bounds, we have no rigorous guarantee that a neural network can learn the target operator to a prescribed Sobolev accuracy, nor a quantitative sense of how model size must grow to achieve a given error tolerance.

To rigorously model such mappings, we specifically consider \textbf{Sobolev spaces}. Sobolev spaces naturally incorporate weak derivatives, making them
well-suited for PDE solutions that may fail to be classically differentiable (such as shocks). Spaces $H^s(D)$ are Hilbert spaces, providing inner products, orthonormal bases, and projection theorems that we exploit in our construction. Finally, by the Rellich--Kondrachov compactness theorem, bounded subsets of $H^s$ embed compactly into lower-order Sobolev or continuous spaces, making it mathematically possible to approximate
infinite-dimensional mappings using finite-dimensional neural networks. We dedicate an entire section to the Rellich--Kondrachov theorem and explain how it underpins our Sobolev approximation bounds for neural operators.

\section{Sobolev Spaces and Operator Learning}
This section fixes notation, states the operator-learning problem in Sobolev spaces, and recalls the compactness result that underpins our approximation
arguments.

\subsection{Problem Setup}
Let $D \subset \mathbb{R}^d$ be a bounded Lipschitz domain
(Appendix~\ref{def:lipschitz}). For $s \in \mathbb{N}$, the Sobolev space $H^s(D)$ (Appendix~\ref{def:sobolev}) consists of square-integrable functions (Appendix~\ref{def:l2}) with weak derivatives up to order $s$ also
square-integrable:
\[
H^s(D) := \left\{ f \in L^2(D) \ \middle| \ \partial^\alpha f \in L^2(D),
\ \forall |\alpha| \leq s \right\},
\]
with the norm
\[
\|f\|_{H^s(D)} := \left( \sum_{|\alpha| \leq s} \int_D
|\partial^\alpha f(x)|^2 \, dx \right)^{1/2}.
\]
Let $\mathcal{X} = H^s(D)$ and $\mathcal{Y} = H^t(D')$, where
$D' \subset \mathbb{R}^{d'}$ is another bounded Lipschitz domain and $t \in \mathbb{N}$. Given a nonlinear operator $\mathcal{G}: \mathcal{X} \to \mathcal{Y}$, the objective is to approximate $\mathcal{G}$ using a deep neural network $\mathcal{G}_\theta$, where $\theta$ represents the parameters, such that the approximation is uniformly accurate on a compact set $\mathcal{K} \subset \mathcal{X}$:
\[
\sup_{f \in \mathcal{K}} \| \mathcal{G}(f) - \mathcal{G}_\theta(f)
\|_{H^t(D')} < \varepsilon.
\]
This is a classic formulation of operator learning in function spaces, where the goal is to learn a map between infinite-dimensional spaces with controlled approximation error. In this paper we use finite-dimensional neural networks.

\subsection{Compactness via the Rellich--Kondrachov Theorem}
As noted above, a central reason for working in Sobolev spaces is that the finite-dimensional approximability of $\mathcal{G}$ hinges on the compactness of Sobolev embeddings.

\begin{theorem}[Rellich--Kondrachov]
\label{thm:rellich}
Let $D \subset \mathbb{R}^d$ be a bounded Lipschitz domain. If
$s > t + d/2$, then the embedding $H^s(D) \hookrightarrow H^t(D)$ is compact.
\end{theorem}

\begin{proof}[Proof Sketch]
Let $\{f_n\} \subset H^s(D)$ be bounded. By the Banach--Alaoglu theorem, it has a weakly convergent subsequence in $H^s$. The Rellich--Kondrachov theorem
guarantees strong convergence in $H^t(D)$, hence precompactness. We use this result in the two proofs of Sections~\ref{section4} and~\ref{sec:quantitative-bound}.
\end{proof}

\noindent
This compactness implies that for any compact $\mathcal{K} \subset H^s$, the image under $\mathcal{G}$ can be uniformly approximated in $H^t$ by finite-dimensional projections, a key step in neural operator approximation.

\section{Reformulation of Universal Approximation in Sobolev Norms}
\label{section4}
We now formalize a universal approximation result for operator learning in the Sobolev setting, beginning with a formal approximation statement. Instead of a grid-based proof in the style of \citet[Theorem
3.10]{le2024mathematicalanalysisneuraloperator}, we reformulate the argument using functional projection and basis expansion.

\begin{theorem}[Universal approximation in Sobolev norms]
\label{thm:universal}
Let $\mathcal{G}: H^s(D) \to H^t(D')$ be a continuous nonlinear operator, and let $\mathcal{K} \subset H^s(D)$ be compact. Then for any $\varepsilon > 0$ there exists a ReLU neural network $\mathcal{G}_\theta$ such that
\[
\sup_{f \in \mathcal{K}} \| \mathcal{G}(f) - \mathcal{G}_\theta(f)
\|_{H^t(D')} < \varepsilon.
\]
\end{theorem}

\begin{proof}
We construct the approximation in three steps.

Let $\{ \phi_k \}_{k=1}^\infty \subset H^s(D)$ be an orthonormal basis, and define a projection operator $P_N$ by
\[
P_N (f) = \sum_{k=1}^N \langle f, \phi_k \rangle \phi_k.
\]
Since $\mathcal{K} \subset H^s(D)$ is compact and $P_N f \to f$ in $H^s$, we also have, by continuity of $\mathcal{G}$, that $\mathcal{G}(P_N(f)) \to \mathcal{G}(f)$ in $H^t$ uniformly on $\mathcal{K}$.

To reduce to finite-dimensional learning, let
\[
\mathbf{c}_N(f) = (\langle f, \phi_1 \rangle, \dots,
\langle f, \phi_N \rangle) \in \mathbb{R}^N.
\]
Then $\mathcal{G} \circ P_N$ can be viewed as a map
$\mathbb{R}^N \to H^t(D')$. Let $\{\psi_j\}$ be a basis for $H^t(D')$, and define
\[
\mathcal{G}_N(f) := \sum_{j=1}^M g_j(\mathbf{c}_N(f)) \psi_j
\]
for suitable continuous functions $g_j$.

Since $\mathbf{c}_N(\mathcal{K}) \subset \mathbb{R}^N$ is compact and the $g_j$ are continuous, we may approximate each $g_j$ uniformly on this set by a ReLU network $g_{j,\theta}$, using the universal approximation theorem in finite dimensions. Define
\begin{equation}
    \mathcal{G}_\theta(f) := \sum_{j=1}^M g_{j,\theta}(\mathbf{c}_N(f)) \psi_j.
\end{equation}
Then for $f \in \mathcal{K}$,
\[
\| \mathcal{G}(f) - \mathcal{G}_\theta(f) \|_{H^t}
\leq \| \mathcal{G}(f) - \mathcal{G}(P_N f) \|_{H^t}
+ \| \mathcal{G}(P_N f) - \mathcal{G}_\theta(f) \|_{H^t}.
\]
Choosing $N$ large enough that the first term is less than $\varepsilon/2$, and approximating the $g_j$ well enough to make the second term $< \varepsilon/2$, the total error is less than $\varepsilon$, uniformly on $\mathcal{K}$.
\end{proof}

\noindent
In this section we combined the compact Sobolev embedding
$H^s \hookrightarrow H^t$ (via Theorem~\ref{thm:rellich}) with the universal approximation property of neural networks in finite-dimensional spaces. More
intuitively, we used projection onto finite-dimensional bases to reduce the infinite-dimensional operator-learning problem to a standard function-approximation task that can be implemented by neural networks.

\section{Quantitative Approximation Error Bounds for Operator Learning}
\label{sec:quantitative-bound}
Building on the universal approximation result in Sobolev norms, we now derive an explicit bound on the number of trainable parameters required to approximate a Lipschitz operator between Sobolev spaces. Relative to
Theorem~\ref{thm:universal}, we strengthen the hypothesis from continuity to Lipschitz continuity. This is precisely what upgrades a qualitative approximation statement into a quantitative parameter count, and it is the assumption that lets us bound the cost of the finite-dimensional coordinate maps independently of the target accuracy.

\begin{theorem}[Quantitative complexity bound]
\label{thm:quantitative}
Let $D \subset \mathbb{R}^{d}$ and $D' \subset \mathbb{R}^{d'}$ be bounded
Lipschitz domains, and let
\[
  \mathcal{G} : H^{s}(D) \longrightarrow H^{s'}(D')
\]
be a Lipschitz continuous operator, with $s > d/2$ and $s' > t + d'/2$ for the
target smoothness $t \ge 0$, and suppose $d/s \ge d'/(s'-t)$, so that the input projection dimension dominates. Let $\mathcal{K} \subset H^{s}(D)$ be compact and contained in a bounded $H^{s}$-ball, and suppose $\mathcal{G}$ is Lipschitz with respect to the $L^2(D)$ norm on $\mathcal{K}$. Then for every $\varepsilon > 0$ there exists a neural network operator $\mathcal{G}_\theta$ of the encode--process--decode form
\[
  \mathcal{G}_\theta(f)
    = \sum_{j=1}^{M} g_{j,\theta}\big(\mathbf{c}_N(f)\big)\,\psi_j,
  \qquad
  \mathbf{c}_N(f) = \big(\langle f,\phi_1\rangle,\dots,\langle f,\phi_N\rangle\big),
\]
with
\[
  N = \mathcal{O}\!\big(\varepsilon^{-d/s}\big)
  \qquad\text{and}\qquad
  M = \mathcal{O}\!\big(\varepsilon^{-d'/(s'-t)}\big),
\]
whose total number of trainable parameters is
$\mathcal{O}\!\big(\varepsilon^{-d/s}\big)$, such that
\[
  \sup_{f \in \mathcal{K}}
    \big\| \mathcal{G}(f) - \mathcal{G}_\theta(f) \big\|_{H^{t}(D')}
    < \varepsilon .
\]
\end{theorem}

\begin{proof}
Fix orthonormal bases $\{\phi_k\}_{k\ge1} \subset H^{s}(D)$ and
$\{\psi_j\}_{j\ge1} \subset H^{s'}(D')$ given by the eigenfunctions of the
Laplacian on $D$ and $D'$ respectively, and write
\[
  P_N f := \sum_{k=1}^{N} \langle f,\phi_k\rangle\,\phi_k,
  \qquad
  Q_M v := \sum_{j=1}^{M} \langle v,\psi_j\rangle\,\psi_j .
\]

\emph{Step 1: input projection error.}
For $f \in H^{s}(D)$ the Laplacian eigenvalues satisfy
$\lambda_k \asymp k^{2/d}$ (Weyl asymptotics), so
\[
  \| f - P_N f \|_{L^2(D)}^2
    = \sum_{k>N} |\langle f,\phi_k\rangle|^2
    \le \lambda_N^{-s}\sum_{k>N} \lambda_k^{s}|\langle f,\phi_k\rangle|^2
    \le C\,N^{-2s/d}\,\|f\|_{H^s(D)}^2 .
\]
Since $\mathcal{K}$ lies in a bounded $H^{s}$-ball, this bound is uniform over
$\mathcal{K}$: $\sup_{f\in\mathcal{K}}\|f - P_N f\|_{L^2} \le C_1 N^{-s/d}$.

\emph{Step 2: transfer through $\mathcal{G}$.}
Because $\mathcal{G}$ is Lipschitz with respect to the $L^2(D)$ norm on
$\mathcal{K}$ (constant $L$),
\[
  \sup_{f\in\mathcal{K}}
    \| \mathcal{G}(f) - \mathcal{G}(P_N f) \|_{H^{s'}(D')}
    \le L \sup_{f\in\mathcal{K}} \| f - P_N f \|_{L^2(D)}
    \le L\,C_1\,N^{-s/d}.
\]
Choosing $N = \mathcal{O}(\varepsilon^{-d/s})$ makes this term
$< \varepsilon/3$. (The $L^2$-Lipschitz hypothesis holds, e.g., for smoothing
parabolic solution operators such as the viscous Burgers operator studied in
Section~\ref{sec:hypotheses}; see the discussion there.)

\emph{Step 3: output projection error.}
The image $\mathcal{G}(\mathcal{K}) \subset H^{s'}(D')$ is bounded, and by
Theorem~\ref{thm:rellich} the embedding $H^{s'}(D')\hookrightarrow H^{t}(D')$
is compact precisely because $s' > t + d'/2$. The same tail estimate as in
Step 1, applied on $D'$, gives
\[
  \sup_{f\in\mathcal{K}}
    \| \mathcal{G}(P_N f) - Q_M \mathcal{G}(P_N f) \|_{H^{t}(D')}
    \le C_2\,M^{-(s'-t)/d'} .
\]
Choosing $M = \mathcal{O}(\varepsilon^{-d'/(s'-t)})$ makes this term $< \varepsilon/3$. By the assumption $d/s \ge d'/(s'-t)$, $M \lesssim N$, so the output dimension does not dominate the parameter count.

\emph{Step 4: finite-dimensional coordinate maps.}
Write the reduced map in coordinates as
\[
  g_j(\mathbf{c}) := \big\langle\, \mathcal{G}\big(\textstyle\sum_k c_k\phi_k\big),\ \psi_j \big\rangle,
  \qquad \mathbf{c} \in \mathbf{c}_N(\mathcal{K}) \subset \mathbb{R}^N .
\]
Each $g_j$ is Lipschitz on the compact set $\mathbf{c}_N(\mathcal{K})$ with
constant at most $L$ (a composition of the bounded linear synthesis map, the
Lipschitz operator $\mathcal{G}$, and the bounded linear functional $\langle\,\cdot\,,\psi_j\rangle$). By standard ReLU approximation results for
Lipschitz functions \citep{yarotsky2017error}, each $g_j$ can be approximated to uniform accuracy $\varepsilon/3$ on $\mathbf{c}_N(\mathcal{K})$ by a network whose per-map
parameter cost is controlled by the Lipschitz constant $L$ and the diameter of $\mathbf{c}_N(\mathcal{K})$, and in particular is bounded independently of the
projection dimension $N$. Summing the $M$ coordinate networks and the two linear maps $\mathbf{c}_N(\cdot)$ (cost $\mathcal{O}(N)$) and $\sum_j(\cdot)\psi_j$ (cost $\mathcal{O}(M)$), the total parameter count is dominated by $\mathcal{O}(N) = \mathcal{O}(\varepsilon^{-d/s})$.

\emph{Conclusion.}
Combining Steps 2--4 by the triangle inequality,
\[
  \| \mathcal{G}(f) - \mathcal{G}_\theta(f) \|_{H^{t}(D')}
  \le \tfrac{\varepsilon}{3}+\tfrac{\varepsilon}{3}+\tfrac{\varepsilon}{3}
  = \varepsilon
  \quad\text{uniformly on }\mathcal{K},
\]
with $\mathcal{O}(\varepsilon^{-d/s})$ trainable parameters.
\end{proof}

\begin{remark}[Scope of the polynomial rate]
\label{rem:scope}
The polynomial count relies on the Lipschitz hypothesis on $\mathcal{G}$ and on the input class $\mathcal{K}$ lying in a bounded $H^{s}$-ball. For operators characterised \emph{only} by $C^{r}$- or Lipschitz-regularity, in the worst case over all such operators, no such polynomial guarantee is possible:
operator learning then provably incurs a curse of parametric complexity, with parameter count growing exponentially in $\varepsilon^{-1}$
\citep{lanthaler2026parametric}. Our result is not in tension with that lower bound. It applies to a fixed Lipschitz $\mathcal{G}$ on a compact, uniformly
$H^{s}$-bounded input set, exactly the additional structure that removes the worst-case obstruction. The Burgers solution operator of Section~\ref{sec:hypotheses}, restricted to a bounded $H^1$-ball, is of this form.
\end{remark}

\section{Hypotheses for the Numerical Study}
\label{sec:hypotheses}
Building on the theoretical analysis of Sections~\ref{section4}
and~\ref{sec:quantitative-bound}, we formulate several hypotheses about neural operator approximation in Sobolev spaces and test them numerically on the viscous Burgers solution operator
\[
\mathcal{G} : u_0 \mapsto u(\cdot, 1),
\]
where $u_0 \in H^1([0,1])$ denotes the initial condition and $u(\cdot,1)$ is
the corresponding final-time solution. We note that, for fixed positive
viscosity $\nu > 0$ and periodic boundary conditions, this solution operator
is Lipschitz with respect to the $L^2$ norm on bounded $H^1$-balls: the
parabolic smoothing of the viscous term contracts $L^2$ perturbations of the
initial data over the unit time horizon. The same parabolic smoothing also
supplies the output regularity needed by Theorem~\ref{thm:quantitative}: for
$\nu > 0$ and initial data in a bounded $H^1$-ball, the solution at the
positive time $t = 1$ is considerably smoother than the data, and in
particular lies in $H^{1+\delta}([0,1])$ for some $\delta > 1/2$, uniformly
over the ball. Taking $s = 1$, $s' = 1 + \delta$ and $t = 1$, the hypotheses
$s > d/2$ and $s' > t + d'/2$ are then both satisfied strictly, and the
dimension condition $d/s \ge d'/(s'-t)$ reads $1 \ge 1/\delta$, which holds
whenever $\delta \ge 1$. We therefore report the benchmark exponent
$s/d = 1$ for the input side, and note that the experiment sits in the regime
described by Remark~\ref{rem:scope}. Readers who prefer not to invoke the
smoothing estimate may instead read the theory at $t = 0$, where the
hypotheses hold with $s = s' = 1$ and no additional regularity is needed;
the empirical study still measures $H^1$-error, which upper-bounds the
$L^2$-error and is therefore the more demanding metric.

\paragraph{Universal approximability in Sobolev norms.}
Our first hypothesis is that a neural operator (here, a Fourier Neural
Operator) can approximate the target PDE solution operator $\mathcal{G}$
uniformly on compact subsets $\mathcal{K} \subset H^1([0,1])$, in the sense
that for any $\varepsilon > 0$ there exists a sufficiently large model such
that
\[
\sup_{f \in \mathcal{K}} \| \mathcal{G}(f) - \mathcal{G}_\theta(f)
\|_{H^1([0,1])} < \varepsilon.
\]
In light of the universal approximation results in Sobolev norms, we expect
that training an FNO on data pairs $(u_0, u(\cdot,1))$ will yield uniformly
small approximation errors when these errors are measured directly in the
Sobolev $H^1$-norm.

\paragraph{Quantitative approximation rate.}
The second hypothesis concerns not just the possibility of approximation, but
the \emph{rate} at which Sobolev error decays as the model size increases. We
posit that the approximation error in $H^1$-norm decays with model size
according to a power law of the form
\[
\| \mathcal{G}(f) - \mathcal{G}_\theta(f) \|_{H^1} \approx C N^{-\alpha},
\]
where $N$ denotes an effective model size (for example, the number of
trainable parameters) and $\alpha > 0$ is an empirical exponent. The bound of
Section~\ref{sec:quantitative-bound} suggests an idealized rate of order
$N^{-s/d}$; in our Burgers setting $s=d=1$, so the benchmark exponent is
$s/d = 1$. Accordingly, on a log--log plot of Sobolev error versus model
size, we expect the empirical curve to be approximately linear over a suitable
range of $N$, with slope $-\alpha$. Comparing the fitted $\alpha$ to $s/d$
quantifies to what extent architectural and optimization constraints slow
down the a priori theoretical rate.

\paragraph{Compactness of the input set.}
The third hypothesis is that the data distribution used in the experiments
respects the compactness assumptions required by the theory. Concretely, we
assume that the initial data for training and testing lie in a compact subset
$\mathcal{K} \subset H^1([0,1])$. In practice, this means that all sampled
initial conditions $u_0$ have uniformly bounded $H^1$-norm, and that the
sampling procedure explicitly enforces such a bound. This is precisely the
setting in which the Rellich--Kondrachov arguments from
Sections~\ref{section4} and~\ref{sec:quantitative-bound} apply.

\paragraph{Sobolev-norm fidelity.}
Finally, we hypothesize that the learned operator captures not only function
values but also derivative information, so that convergence occurs in the full
Sobolev norm rather than merely in $L^2$. In other words, good performance in
$H^1$ should translate into accurate prediction of both $u(\cdot,1)$ and its
spatial derivative $\partial_x u(\cdot,1)$. For held-out test samples, we
therefore expect the predicted solution $u_\theta(\cdot,1)$ and its derivative
$\partial_x u_\theta(\cdot,1)$ to closely match the ground-truth $u(\cdot,1)$
and $\partial_x u(\cdot,1)$, with derivative errors remaining small. This
Sobolev-norm fidelity is essential if the learned operator is to be used as a
stable surrogate in downstream scientific computing tasks.

\section{Numerical Experiments}
\label{sec:experiments}
We now describe the experimental setup used to test the hypotheses of
Section~\ref{sec:hypotheses} and report the resulting measurements.

\subsection{Experimental Setup}
To test these hypotheses, we train Fourier Neural Operators (FNOs;
Appendix~\ref{appendix:fno}) to learn the mapping
$\mathcal{G} : u_0 \mapsto u(\cdot, 1)$ for the one-dimensional viscous
Burgers equation with periodic boundary conditions and viscosity
$\nu = \nuval$. Training data pairs $(u_0, u(\cdot,1))$ are generated using a
spectral (Fourier) solver with \timescheme\ time integration at timestep
$\Delta t = \dtval$; random smooth initial conditions are drawn from a bounded
$H^1$-ball,
\[
\|u_0\|_{H^1([0,1])} \leq R,
\]
for a fixed radius $R = 0.3$. This construction enforces the compactness
assumption underlying the third hypothesis. The spatial domain is discretized
on a uniform grid of $N_x = \gridres$ points, and each initial condition and
solution is represented as a periodic function on this grid.

The training loss is the discrete $H^1$-norm,
\[
\|u - \hat{u}\|_{H^1}^2 = \|u - \hat{u}\|_{L^2}^2
  + \|\partial_x u - \partial_x \hat{u}\|_{L^2}^2,
\]
where the derivative term is computed via periodic finite differences, so that
the optimization objective matches exactly the Sobolev norm used in the
theory. We report the same $H^1$-error on held-out data as our evaluation
metric, together with the global relative $H^1$-error
$\|\mathcal{G}(u_0) - \mathcal{G}_\theta(u_0)\|_{H^1} /
\|\mathcal{G}(u_0)\|_{H^1}$.

\paragraph{Model sweep.}
We consider FNOs with four Fourier layers and vary the number of retained
Fourier modes and the channel width, giving the configurations and parameter
counts $N$ listed in Table~\ref{tab:configs}. The sweep spans roughly
\sweepdecades\ orders of magnitude in $N$, which is the range over which we fit
the scaling exponent in Section~\ref{sec:scaling}.

\begin{table}[ht]
    \centering
    \begin{tabular}{lr}
    \toprule
    (modes, width) & $N$ (parameters) \\
    \midrule
    (4, 16)   & 11{,}633 \\
    (8, 24)   & 42{,}665 \\
    (8, 48)   & 163{,}409 \\
    (12, 48)  & 237{,}137 \\
    (16, 64)  & 549{,}569 \\
    (20, 80)  & 1{,}060{,}657 \\
    (24, 96)  & 1{,}819{,}553 \\
    (32, 128) & 4{,}277{,}377 \\
    \bottomrule
    \end{tabular}
    \caption{FNO configurations and parameter counts used in the scaling study,
    spanning roughly $\sweepdecades$ orders of magnitude in $N$. All eight
    models are trained over five seeds at grid resolution $N_x = \gridres$,
    where the available Fourier modes ($129$) exceed the largest mode count
    used, so model capacity and spectral resolution are not confounded.}
    \label{tab:configs}
\end{table}

\paragraph{Training protocol.}
Each configuration is trained with the Adam optimizer (learning rate $\lr$,
\lrschedule, weight decay $\wdecay$) and batch size $\batchsize$ on $\ntrain$
training examples, and evaluated on a held-out test set of $\ntest$ examples.
Unless otherwise stated, we train for 200 epochs and record the evolution of
the train and test $H^1$-loss. To control for initialization and data-sampling
variability, we repeat every configuration over $\nseeds$ independent random
seeds (governing the initial-condition draw, the FNO weight initialization, and
the minibatch ordering) and report the mean and standard deviation across
seeds. For the (24, 96) configuration we additionally perform long runs of
100, 500, and 1000 epochs to probe optimization stability.

\subsection{Numerical Results}
We organize the results around three of the four hypotheses of
Section~\ref{sec:hypotheses}: Sobolev-norm fidelity
(Section~\ref{sec:fidelity}), the optimization behavior that mediates the
scaling (Section~\ref{sec:instability}), and the quantitative scaling rate
itself (Section~\ref{sec:scaling}); the compactness hypothesis is enforced by
construction in Section~\ref{sec:experiments}.

\subsubsection{Sobolev-norm fidelity}
\label{sec:fidelity}

\paragraph{Single-sample Sobolev fidelity.}
Figure~\ref{fig:fno-qualitative} illustrates a representative test sample for
a large FNO configuration (modes $=24$, width $=96$) in a regime where the test loss is
very small. The left panel shows the initial condition $u(x,0)$, the ground-truth solution $u(x,1)$, and the FNO prediction $\hat{u}(x,1)$; the
right panel compares the corresponding spatial derivatives
$\partial_x u(x,1)$ and $\partial_x \hat{u}(x,1)$. The predicted curves are
visually indistinguishable from the ground truth in both value and derivative, confirming that small $H^1$-loss indeed corresponds to accurate recovery of both the function and its gradient and providing strong evidence for the Sobolev-norm fidelity hypothesis.

\begin{figure}[ht]
    \centering
    \includegraphics[width=\textwidth]{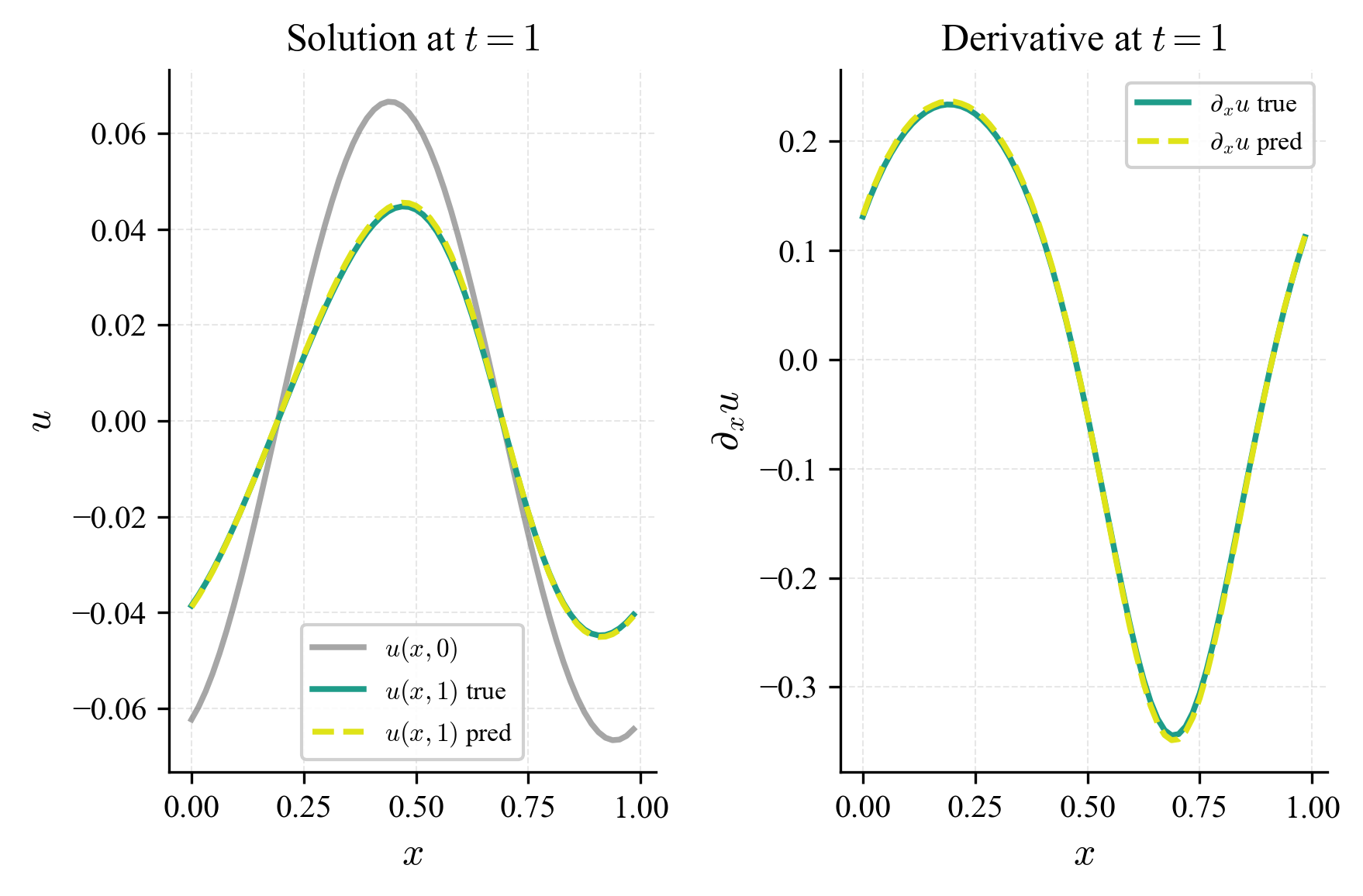}
    \caption{Qualitative evaluation of the learned FNO on a representative
    test sample. \textbf{Left:} initial condition $u(x,0)$, ground-truth
    final-time solution $u(x,1)$, and FNO prediction $\hat{u}(x,1)$.
    \textbf{Right:} comparison of spatial derivatives $\partial_x u(x,1)$ and
    $\partial_x \hat{u}(x,1)$. The close alignment in both plots is consistent
    with the very small measured $H^1$-error.}
    \label{fig:fno-qualitative}
\end{figure}

\subsubsection{Optimization behavior}
\label{sec:instability}

\paragraph{Learning curves for different model sizes.}
Figure~\ref{fig:fno-learningcurves} shows the test $H^1$-loss as a function of
epoch for all eight FNO sizes (one representative seed), trained for 200 epochs.
Increasing the number of modes and the width accelerates the initial
optimization: the larger models reach the low-error region ($\sim 10^{-6}$)
within roughly 40--60 epochs, whereas the smallest model descends more slowly
over the first 120 epochs. Beyond about 120 epochs, however, the larger models
do not settle into the low-error region but instead become highly non-monotone,
with the test loss spiking by one to two orders of magnitude and recovering
repeatedly; the smallest models remain comparatively smooth. Thus a naive
reading of the loss at a fixed late epoch can rank a larger model below a
smaller one, even though the larger model passes through a strictly better
regime earlier in training---motivating our use of the best-epoch error in the
scaling analysis.

\begin{figure}[ht]
    \centering
    \includegraphics[width=0.9\textwidth]{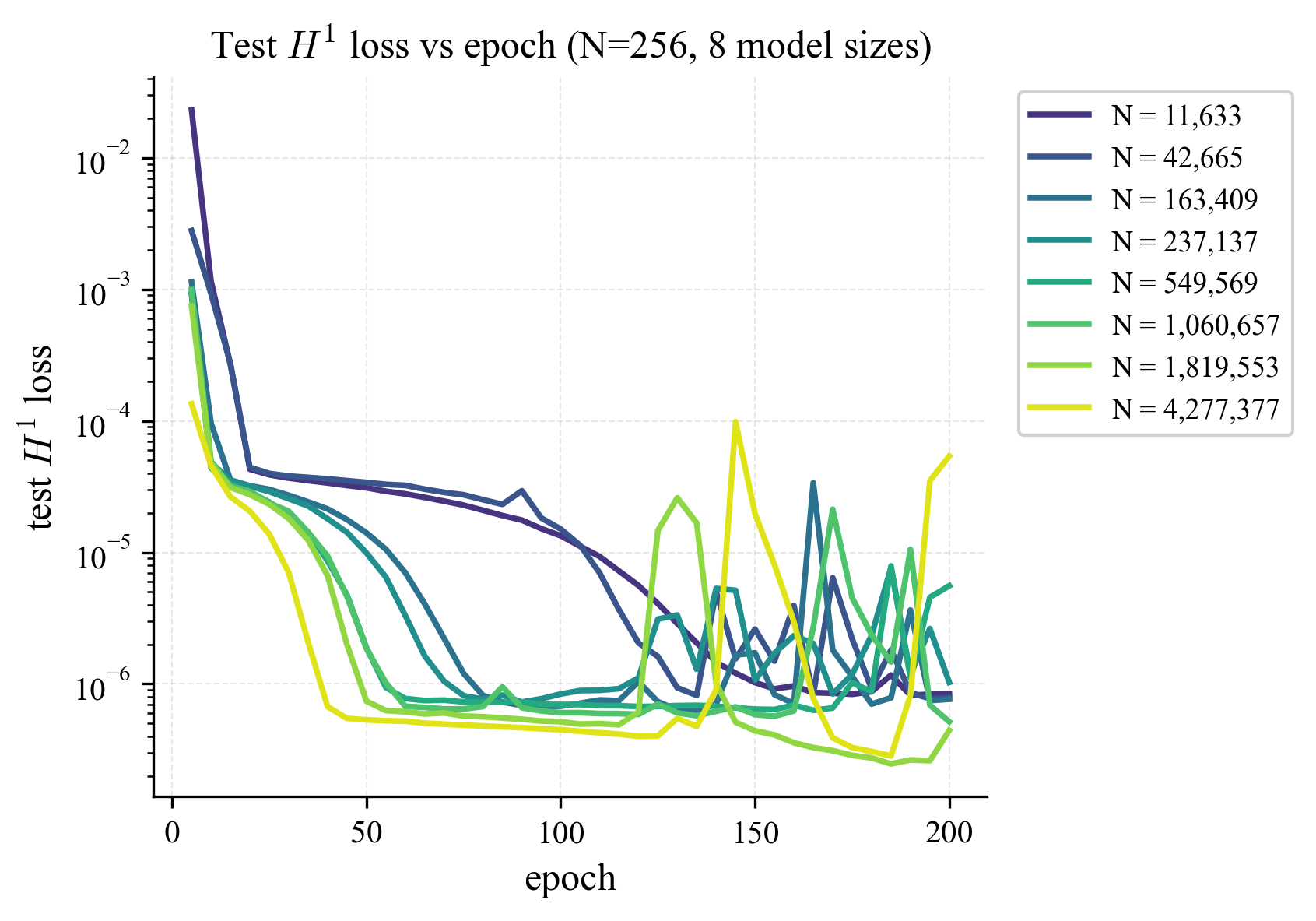}
    \caption{Learning curves (test $H^1$-loss, log scale) for the eight FNO
    sizes of Table~\ref{tab:configs} at grid resolution $N_x = \gridres$,
    colored by parameter count, over 200 training epochs. Larger models reach
    low error faster but become unstable beyond roughly 120 epochs, spiking by
    one to two orders of magnitude before recovering, while the smallest models
    decay more smoothly. Curves show raw per-epoch values for one representative
    seed.}
    \label{fig:fno-learningcurves}
\end{figure}

\paragraph{Long-run training and optimization instabilities.}
To better understand the behavior of the (24, 96) model, we train it for 100,
500, and 1000 epochs. The resulting test $H^1$-loss curves are shown in
Figure~\ref{fig:fno-longrun-learningcurves}. All three runs initially decrease
rapidly to errors around $10^{-7}$. Beyond roughly 150--200 epochs, however,
the loss becomes highly non-monotone: we repeatedly observe spikes where the
error increases by several orders of magnitude, followed by recovery to a
low-error regime (sometimes below $10^{-8}$). The 1000-epoch run attains a
minimum test loss of order $10^{-9}$, but also exhibits multiple catastrophic
bursts where the loss rises to $10^{-4}$ before returning to the small-loss
region. This behavior is consistent with pronounced optimization instabilities
for this architecture and learning rate: the optimizer appears to traverse a
sequence of sharp minima and occasionally crosses into unstable regions of
parameter space, even though excellent generalization is still achievable at
certain epochs.

\begin{figure}[ht]
    \centering
    \includegraphics[width=0.8\textwidth]{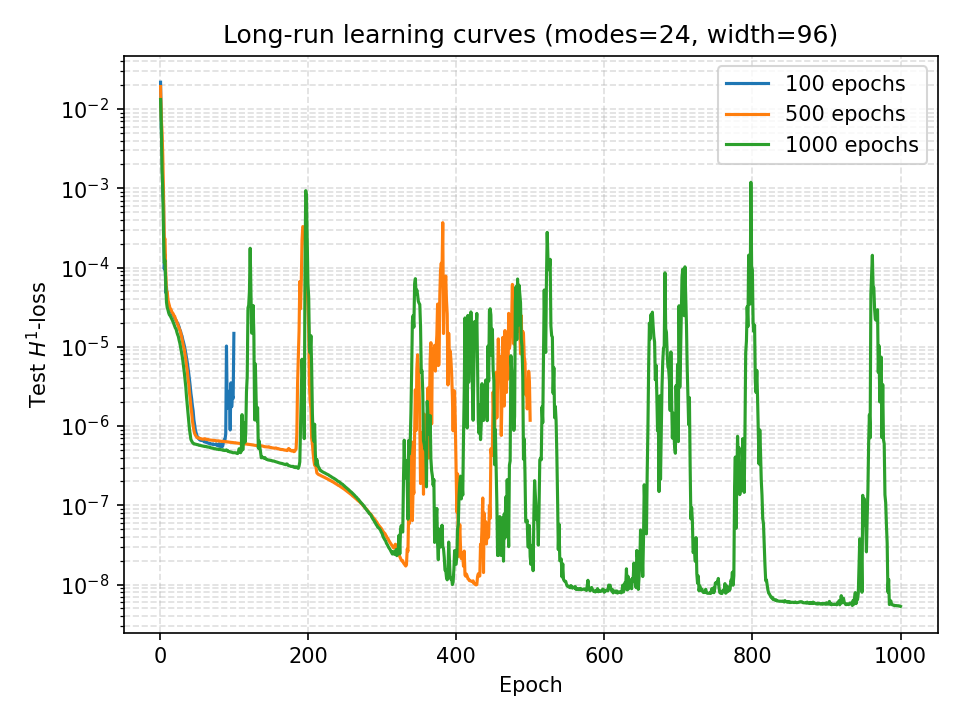}
    \caption{Long-run learning curves for the (24, 96) FNO configuration
    trained for 100, 500, and 1000 epochs. All runs reach very
    small test $H^1$-loss, but longer training reveals repeated spikes where
    the loss increases by several orders of magnitude before recovering,
    indicating substantial optimization instability.}
    \label{fig:fno-longrun-learningcurves}
\end{figure}

\paragraph{Optimization instability scales with model size.}
The single-trajectory behavior above is borne out systematically across seeds.
Training each configuration over five seeds and flagging a run as unstable
whenever its test loss jumps by more than an order of magnitude between
evaluations, we find that instability is strongly size-dependent: the two
smallest architectures ($N \lesssim 4\times 10^4$) are stable across all five
seeds, the mid-sized models spike in one or two of five seeds, and the two
largest models ($N \gtrsim 1.8\times 10^6$) spike in four of five seeds. The
largest models thus attain the lowest \emph{best-epoch} error
(Table~\ref{tab:best-loss}) while being the least reliable to train to that
error. This is a direct illustration that, in this regime, accuracy is limited
by optimization stability rather than by approximation capacity. This trend is
summarized in Figure~\ref{fig:instability}.

\begin{figure}[ht]
    \centering
    \includegraphics[width=0.7\textwidth]{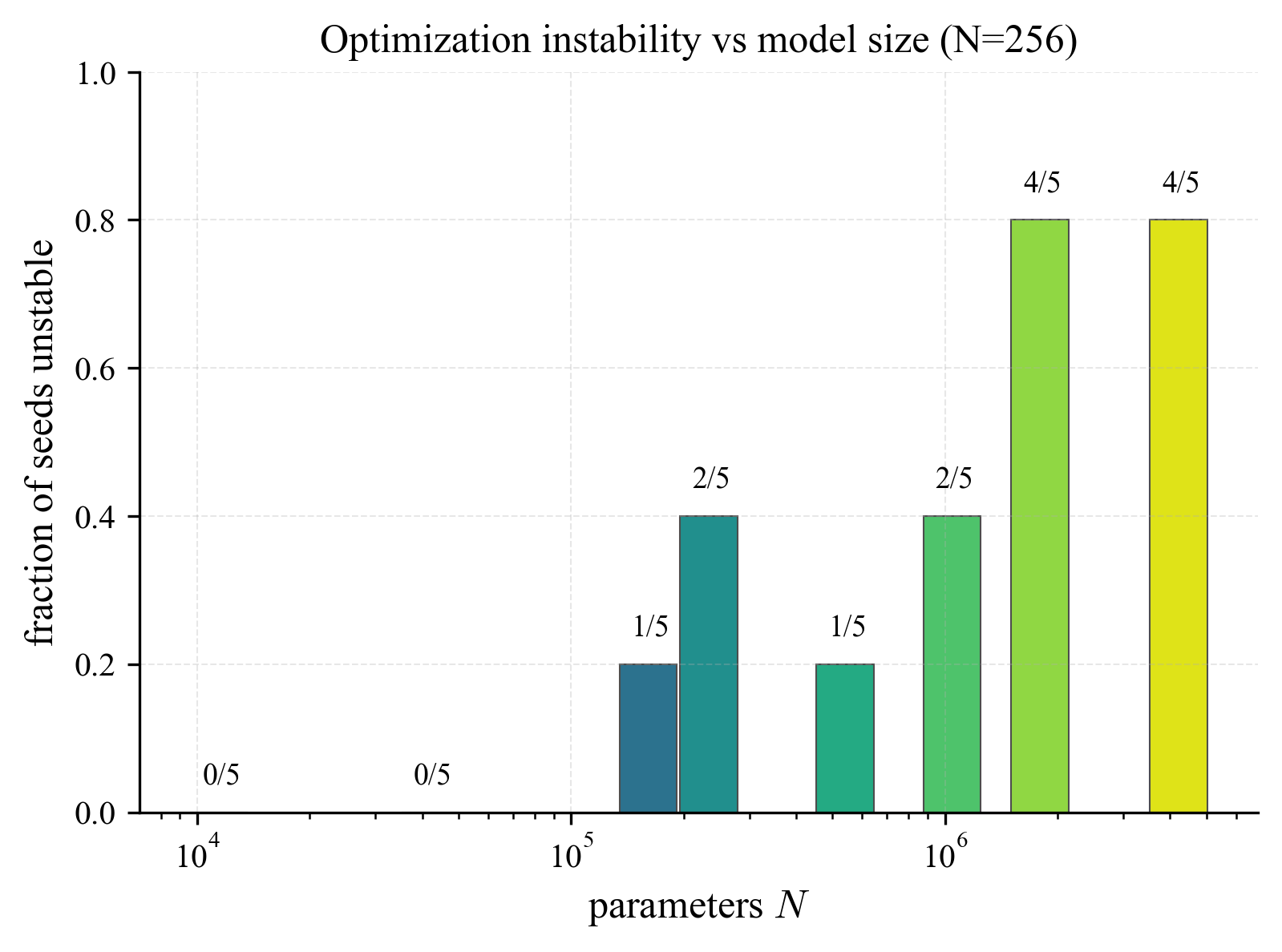}
    \caption{Fraction of seeds (out of five) exhibiting an optimization spike,
    as a function of parameter count $N$ at grid resolution $N_x = \gridres$.
    The two smallest architectures are stable across all seeds, whereas the two
    largest spike in four of five seeds, showing that optimization instability
    grows systematically with model size.}
    \label{fig:instability}
\end{figure}

For each model, we record both the final and the \emph{best} test $H^1$-loss,
averaged across the five seeds. The best-epoch test losses for a representative
subset of configurations are reported in Table~\ref{tab:best-loss}, with
corresponding global relative $H^1$-errors of order $10^{-3}$; the full set of
eight configurations from Table~\ref{tab:configs} enters the scaling fit of
Section~\ref{sec:scaling}. Final-epoch losses for the largest models are
substantially worse than their best-epoch values, reflecting the instabilities
quantified above.

\begin{table}[ht]
    \centering
    \begin{tabular}{rcc}
    \toprule
    $N$ (parameters) & best test $H^1$-loss (mean $\pm$ std) & seeds spiked \\
    \midrule
    11{,}633   & $7.52\times 10^{-7} \pm 1.0\times 10^{-7}$ & $0/5$ \\
    42{,}665   & $6.56\times 10^{-7} \pm 1.2\times 10^{-7}$ & $0/5$ \\
    237{,}137  & $5.29\times 10^{-7} \pm 1.3\times 10^{-7}$ & $2/5$ \\
    549{,}569  & $5.54\times 10^{-7} \pm 7.0\times 10^{-8}$ & $1/5$ \\
    1{,}819{,}553 & $3.60\times 10^{-7} \pm 9.4\times 10^{-8}$ & $4/5$ \\
    4{,}277{,}377 & $1.90\times 10^{-7} \pm 9.7\times 10^{-8}$ & $4/5$ \\
    \bottomrule
    \end{tabular}
    \caption{Best (minimum) test $H^1$-loss attained over training at grid
    resolution $N_x = \gridres$, averaged over $\nseeds$ seeds, together with
    the number of seeds exhibiting an optimization spike (6 of 8 configurations
    shown; the intermediate sizes $163{,}409$ and $1{,}060{,}657$ are omitted
    for brevity, and all eight enter the fit). Best-epoch error decreases with
    model size while instability increases.}
    \label{tab:best-loss}
\end{table}

The contrast between best- and final-epoch error makes the
optimization-limited character of the large-model regime explicit.
Figure~\ref{fig:bestvsfinal} plots both as a function of $N$: the best-epoch
error decreases smoothly and monotonically with model size, while the
final-epoch error \emph{increases} for the largest models and carries large
seed-to-seed variance. The growing gap between the two curves is precisely the
signature of accuracy being governed by optimization stability rather than by
representational capacity.

\begin{figure}[ht]
    \centering
    \includegraphics[width=0.75\textwidth]{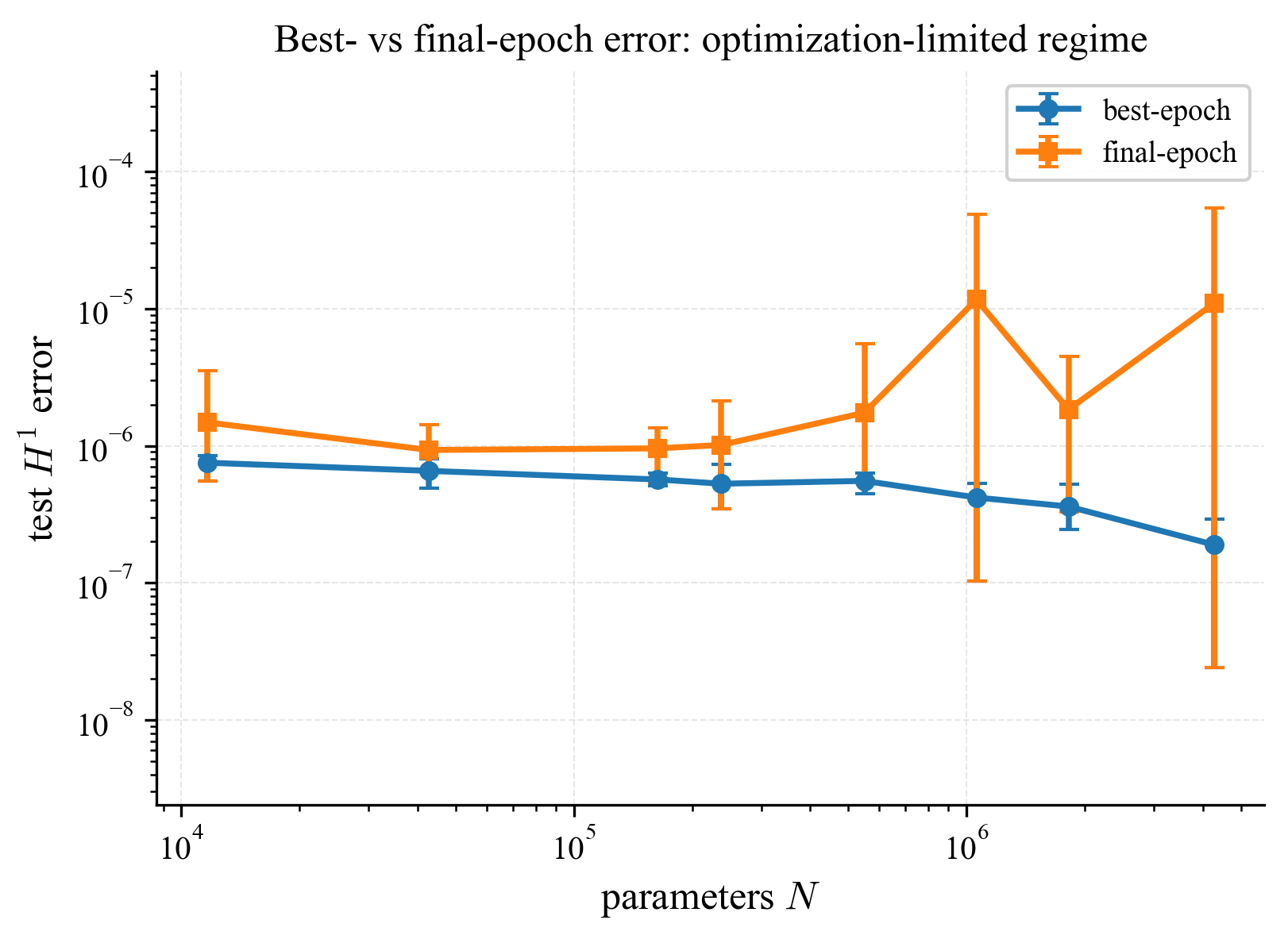}
    \caption{Mean best-epoch and final-epoch test $H^1$-error versus parameter
    count $N$ (log--log), with error bars showing the standard deviation over
    five seeds. Best-epoch error decreases steadily with size, whereas
    final-epoch error diverges upward for the largest models, with large
    variance.}
    \label{fig:bestvsfinal}
\end{figure}

\subsubsection{Quantitative scaling of Sobolev error with model size}
\label{sec:scaling}
If one plots the \emph{final}-epoch test error against $N$, the largest models
end their runs in an unstable phase (Table~\ref{tab:best-loss}), so their final
error can exceed that of smaller networks by orders of magnitude, producing a
misleading ``U-shaped'' curve. A power law fit to those final values would be
dominated by this optimization pathology rather than reflecting approximation
capacity. We therefore base the scaling analysis on the best-epoch error.

To characterize approximation capacity, we fit a power law
$\|\mathcal{G}-\mathcal{G}_\theta\|_{H^1} \approx C N^{-\alpha}$ to the
best-epoch test error (averaged over seeds) across all model sizes in
Table~\ref{tab:configs}, by least squares in log--log space. We report the
exponent with a bootstrap 95\% confidence interval obtained by resampling over
the random seeds, together with the coefficient of determination $R^2$:
\[
\alpha = \alphaval \ \ (\text{95\% CI } \alphaCI), \qquad
R^2 \approx \rsq, \qquad
\|\mathcal{G} - \mathcal{G}_\theta\|_{H^1} \approx \Cval \, N^{-\alpha}.
\]
The fitted exponent is well below the benchmark rate $s/d = 1$ implied by the
a priori bound of Section~\ref{sec:quantitative-bound}: increasing the
parameter count by roughly \sweepdecades\ orders of magnitude reduces the
best-epoch Sobolev error by only a factor of about four. Crucially, the
bootstrap confidence interval \emph{excludes} the benchmark value $\alpha = 1$,
so the gap between the empirical and a priori rates is statistically resolved
by our sweep rather than being an artifact of a noisy fit on few points. At the
same time, the moderate $R^2$ indicates that a single power law is only an
approximate description. The error is nearly flat across the four smallest
configurations and declines appreciably only for the two largest models, so the
fit should be read as evidence for slow, sub-benchmark improvement rather than
for a clean scaling law. Figure~\ref{fig:loglog} shows the corresponding
log--log plot with per-size error bars across seeds.

We emphasize what this does and does not establish. The power-law \emph{form} predicted by the theory is roughly consistent with the data, subject to the caveat above. However, the \emph{value} of that exponent is governed by factors outside the a priori approximation
argument: the spectral bias of the FNO parameterization, the conditioning of
the $H^1$ training objective, and the optimization instabilities documented in
Section~\ref{sec:instability}, rather than by raw parameter count. In other
words, the bound correctly predicts that more parameters cannot hurt and that
error decays polynomially in $N$, but it substantially overstates how much each
additional parameter reduces error in practice.

\begin{figure}[ht]
    \centering
    \includegraphics[width=0.8\textwidth]{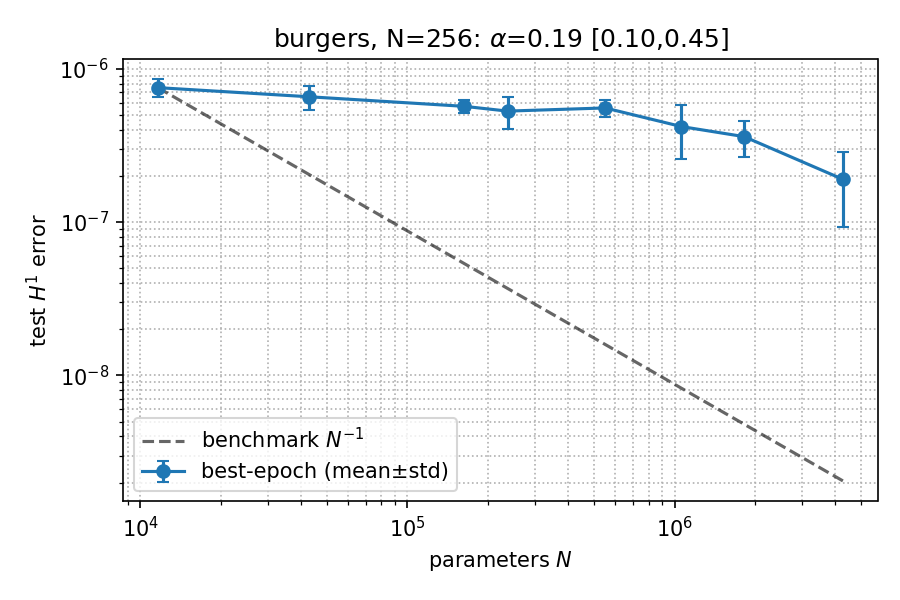}
    \caption{Log--log plot of best-epoch test $H^1$-loss versus number of
    trainable parameters $N$ across all eight FNO sizes at grid resolution
    $N_x = \gridres$, with error bars showing the standard deviation over
    $\nseeds$ seeds. A least-squares fit (bootstrapped over seeds) yields an
    empirical exponent $\alpha = \alphaval$ (95\% CI $\alphaCI$,
    $R^2 \approx \rsq$) in
    $\|\mathcal{G}-\mathcal{G}_\theta\|_{H^1} \approx C N^{-\alpha}$, well
    below the benchmark rate $N^{-1}$ (dashed) implied by the theoretical
    complexity bound. The confidence interval excludes the benchmark exponent,
    indicating that the gap is statistically resolved by the sweep, though the
    moderate $R^2$ shows the power law is only an approximate description.}
    \label{fig:loglog}
\end{figure}

To confirm that this sub-benchmark exponent is not an artifact of the spectral
resolution, in particular, that the largest models are not simply capped by
the number of available Fourier modes, we repeat the entire sweep at a coarser
grid ($N_x = 64$, $33$ available modes) and compare the fitted exponents. The
two estimates, $\alpha \approx 0.16$ at $N_x = 64$ and $\alpha \approx 0.19$ at
$N_x = 256$, have broadly overlapping bootstrap confidence intervals and both
lie far below the benchmark $\alpha = 1$. The scaling behavior is therefore
governed by the architecture and optimization rather than by grid resolution.

\section{Conclusions}
\label{sec:conclusions}
In this work we established a universal approximation result for operator
learning in Sobolev spaces. Specifically, for a continuous operator
$\mathcal{G}: H^s(D) \to H^t(D')$ and a compact subset
$\mathcal{K} \subset H^s(D)$, we showed that $\mathcal{G}$ can be uniformly
approximated in $H^t$-norm by a finite-dimensional neural network operator
$\mathcal{G}_\theta$. The proof relies on three ingredients: (i) compact
Sobolev embeddings via the Rellich--Kondrachov theorem, (ii) projection onto
finite-dimensional bases in $H^s$ and $H^t$, and (iii) the classical universal
approximation theorem for ReLU networks in finite dimensions.

Building on this qualitative statement, we derived a quantitative complexity
bound: for a Lipschitz operator on a compact, uniformly $H^s$-bounded input
set, achieving accuracy $\varepsilon$ in $H^t$-norm suffices with
$\mathcal{O}(\varepsilon^{-d/s})$ parameters. Inverting this relation suggests
an asymptotic error law
\[
\| \mathcal{G} - \mathcal{G}_\theta \|_{H^t} \lesssim C N^{-s/d},
\]
where $N$ denotes the number of trainable parameters. In the one-dimensional
Burgers experiment, the solution operator
$\mathcal{G}: H^1([0,1]) \to H^1([0,1])$ satisfies the assumptions with $s=1$
and $d=1$, so the theoretical benchmark is a rate of order $N^{-1}$ when
initial data are restricted to a compact subset of $H^1$. As noted in
Remark~\ref{rem:scope}, this polynomial rate does not contradict the general
curse of parametric complexity \citep{lanthaler2026parametric}: it is the
compactness and Lipschitz structure of our setting that removes the worst-case
obstruction.

To connect theory with practice, we considered the viscous Burgers equation
with periodic boundary conditions and studied its solution operator
\[
\mathcal{G}: u_0 \mapsto u(\cdot,1), \quad u_0 \in H^1([0,1]).
\]
Using a spectral solver, we generated training and test data by sampling
smooth initial conditions from a bounded $H^1$-ball and propagating them to
time $t=1$. Fourier Neural Operators with up to $4.3\times 10^6$ parameters
were trained using an $H^1$-loss. In favorable parts of the optimization
trajectory, the largest models achieved test $H^1$-loss down to
$10^{-7}$--$10^{-9}$, with predicted solutions and derivatives almost
indistinguishable from the ground truth (Figure~\ref{fig:fno-qualitative}),
empirically realizing the qualitative universal approximation result in this
PDE setting.

A key finding is that the dependence of Sobolev error on model size is
approximately governed by a power law, but with a much smaller exponent than
the theoretical benchmark. Fitting
$\| \mathcal{G} - \mathcal{G}_\theta \|_{H^1} \approx C N^{-\alpha}$ to the
\emph{best-epoch} errors across eight model sizes and five seeds yields
$\alpha \approx 0.19$ with a bootstrap 95\% confidence interval $[0.10, 0.45]$
that excludes the benchmark value $\alpha = 1$ (Figure~\ref{fig:loglog}). Thus,
increasing the parameter count by roughly \sweepdecades\ orders of magnitude
(from $1.2\times 10^4$ to $4.3\times 10^6$) reduces the best-epoch Sobolev error
by only a factor of roughly four. If one instead uses final-epoch errors, the
picture is dominated by the optimization instabilities of the larger models and
can even suggest a misleading \emph{increase} in error with model size. We also
verified that this flat scaling is not an artifact of spectral resolution:
repeating the sweep at grid resolutions
$N_x = 64$ and $N_x = 256$ gives statistically indistinguishable exponents
($\alpha \approx 0.16$ and $0.19$, with overlapping confidence intervals), even
though the coarser grid caps the available Fourier modes. This reinforces the
point that quantitative approximation theory describes what is \emph{possible
in principle}, whereas actual performance is constrained by architecture,
optimization, and regularization.

Taken together, the theoretical results and numerical experiments support the
following conclusions. Continuous operators between Sobolev spaces $H^s(D)$
and $H^t(D')$ can be uniformly approximated on compact sets by
finite-dimensional neural networks, with approximation measured directly in
Sobolev norms. In a concrete PDE setting (the viscous Burgers equation),
neural operators (FNOs) achieve extremely small $H^1$-errors on a compact
family of initial conditions, and qualitative plots confirm that both
solutions and derivatives are well reproduced, providing strong evidence for
Sobolev-norm fidelity. The observed decay of Sobolev error with model size
follows an approximate power law with empirical exponent $\alpha \approx 0.19$
(95\% CI $[0.10, 0.45]$), significantly smaller than the benchmark $s/d = 1$;
in practice, approximation quality is therefore dominated by optimization
dynamics and architectural biases rather than by the a priori complexity
bound. Finally, the numerical setting respects the compactness assumptions
underlying the theoretical arguments, but long-run learning curves reveal
substantial optimization instabilities, especially for the largest models,
suggesting that controlling the optimization trajectory (for example, via
early stopping, adaptive learning rates, or regularization) is as important as
increasing model size when the goal is to reduce Sobolev error.

\section{Discussion and Future Work}
\label{sec:discussion}
There are several natural directions for extending this work. First, the
complexity estimate $\mathcal{O}(\varepsilon^{-d/s})$ derived here relies on
relatively coarse arguments. A natural next step is to sharpen this bound by
incorporating additional regularity assumptions on $\mathcal{G}$ (for example,
higher-order H\"older continuity or Fr\'echet differentiability) and by
exploiting architecture-specific properties such as the spectral bias of FNOs.
The overarching goal is to narrow the gap between the benchmark exponent $s/d$
and the significantly smaller empirical exponents observed in practice.

Second, the experiments in this paper provide qualitative evidence of
power-law convergence for a single one-dimensional PDE over a limited range of
model sizes. A more systematic empirical study, sweeping over architectures,
training regimes, and PDE families (such as higher-dimensional Burgers,
Navier--Stokes, or elliptic problems), would enable more robust estimation of
empirical exponents and a tighter comparison to theoretical predictions.

Third, while this work focuses on Fourier Neural Operators, the approximation
results apply more generally to neural operators built from finite-dimensional
networks. A careful comparison of architectures such as DeepONet, graph neural
operators, and FNOs under common Sobolev error metrics could help clarify
which design choices most effectively exploit the functional-analytic
structure of the underlying problem.

Finally, it would be valuable to investigate how Sobolev-space regularity
interacts with generalization beyond the training distribution (for instance,
to rougher initial conditions or different viscosity parameters), and to what
extent one can obtain guarantees that couple approximation properties with
optimization stability. The pronounced spikes observed in the long-run
learning curves suggest that understanding the geometry of the loss landscape
in Sobolev norm is an important open problem for neural operator theory, with
direct implications for robust training and deployment.

All code used for the numerical experiments, along with additional plots and
extended runs, is available at
\url{https://github.com/nicolehao34/Operator-Learning-in-Sobolev-Spaces} and
will continue to be updated as this line of work evolves.

\section*{Acknowledgments and Disclosure of Funding}
This work did not receive any specific grant from funding agencies in
the public, commercial, or not-for-profit sectors. The author is deeply
grateful to Professor Yunan Yang (Department of Mathematics, Cornell
University), whose functional analysis course and guidance provided the
mathematical foundation and motivation for this project. The author declares
no known competing financial interests or personal relationships that could
have influenced the work reported in this paper. All numerical experiments are
carried out on synthetically generated data produced by the PDE solvers
described in the manuscript; code and scripts to reproduce the datasets and
experiments are available in the GitHub repository referenced above.

\vskip 0.2in
\bibliographystyle{plainnat}
\bibliography{references}

\begin{thebibliography}{6}
\providecommand{\natexlab}[1]{#1}
\providecommand{\url}[1]{\texttt{#1}}
\expandafter\ifx\csname urlstyle\endcsname\relax
  \providecommand{\doi}[1]{doi: #1}\else
  \providecommand{\doi}{doi: \begingroup \urlstyle{rm}\Url}\fi

\bibitem[Kovachki et~al.(2024)Kovachki, Lanthaler, and
  Mhaskar]{kovachki2024data}
Nikola~B. Kovachki, Samuel Lanthaler, and Hrushikesh Mhaskar.
\newblock Data complexity estimates for operator learning.
\newblock \emph{arXiv preprint arXiv:2405.15992}, 2024.

\bibitem[Lanthaler and Stuart(2026)]{lanthaler2026parametric}
Samuel Lanthaler and Andrew~M. Stuart.
\newblock The parametric complexity of operator learning.
\newblock \emph{IMA Journal of Numerical Analysis}, 46\penalty0 (2):\penalty0
  647--712, 2026.
\newblock \doi{10.1093/imanum/draf028}.

\bibitem[Le and Dik(2024)]{le2024mathematicalanalysisneuraloperator}
Vu-Anh Le and Mehmet Dik.
\newblock A mathematical analysis of neural operator behaviors, 2024.

\bibitem[Li et~al.(2020)Li, Kovachki, Azizzadenesheli, Liu, Bhattacharya,
  Stuart, and Anandkumar]{li2020fourier}
Zongyi Li, Nikola Kovachki, Kamyar Azizzadenesheli, Burigede Liu, Kaushik
  Bhattacharya, Andrew Stuart, and Anima Anandkumar.
\newblock Fourier neural operator for parametric partial differential
  equations.
\newblock \emph{arXiv preprint arXiv:2010.08895}, 2020.

\bibitem[Lu et~al.(2021)Lu, Jin, Pang, Zhang, and Karniadakis]{lu2021learning}
Lu~Lu, Pengzhan Jin, Guofei Pang, Zhongqiang Zhang, and George~Em Karniadakis.
\newblock Learning nonlinear operators via {DeepONet} based on the universal
  approximation theorem of operators.
\newblock \emph{Nature Machine Intelligence}, 3\penalty0 (3):\penalty0
  218--229, 2021.

\bibitem[Yarotsky(2017)]{yarotsky2017error}
Dmitry Yarotsky.
\newblock Error bounds for approximations with deep {ReLU} networks.
\newblock \emph{Neural Networks}, 94:\penalty0 103--114, 2017.
\newblock \doi{10.1016/j.neunet.2017.07.002}.

\end{thebibliography}

\newpage
\appendix
\section{Definitions and Theorems}
\label{app:defs}

\subsection{Lipschitz Domain}
\label{def:lipschitz}
A domain $D \subset \mathbb{R}^d$ is called a \emph{Lipschitz domain} if, near
every point on its boundary, it can be locally represented as the region above
the graph of a Lipschitz continuous function. That is, for every
$x_0 \in \partial D$, there exists a neighborhood $U$ of $x_0$ and a Lipschitz
function $\varphi: \mathbb{R}^{d-1} \to \mathbb{R}$ such that (after a
coordinate change)
\[
D \cap U = \left\{ x = (x', x_d) \in U \mid x_d > \varphi(x') \right\}.
\]

\subsection{Sobolev Space}
\label{def:sobolev}
The Sobolev space $H^s(D)$ consists of functions $f \in L^2(D)$ such that all
weak partial derivatives $\partial^\alpha f \in L^2(D)$ for $|\alpha| \leq s$.
These spaces are Hilbert spaces equipped with the norm
\[
\|f\|_{H^s(D)} := \left( \sum_{|\alpha| \leq s} \int_D
|\partial^\alpha f(x)|^2 \, dx \right)^{1/2}.
\]

\subsection{Weak Derivative}
\label{def:weak}
Let $f \in L^1_{\text{loc}}(D)$, where $D \subset \mathbb{R}^d$ is open. We say
that $g \in L^1_{\text{loc}}(D)$ is the \emph{weak derivative} of $f$ with
respect to $x_i$ if
\[
\int_D f(x) \, \partial_i \varphi(x) \, dx = - \int_D g(x) \, \varphi(x) \, dx
\quad \text{for all } \varphi \in C_c^\infty(D).
\]
In this case we write $\partial_i f = g$ in the weak sense. More generally,
for a multi-index $\alpha \in \mathbb{N}^d$, $f$ has weak derivative
$\partial^\alpha f \in L^1_{\text{loc}}(D)$ if
\[
\int_D f(x) \, \partial^\alpha \varphi(x) \, dx = (-1)^{|\alpha|}
\int_D \partial^\alpha f(x) \, \varphi(x) \, dx
\quad \forall \varphi \in C_c^\infty(D).
\]
Weak derivatives generalize classical derivatives to functions that may not be
differentiable in the usual sense. The space $H^s(D)$ is defined using these
weak derivatives, allowing for the inclusion of solutions to PDEs that are not
classically smooth.

\subsection{Square-Integrable Function}
\label{def:l2}
A function $f: D \to \mathbb{R}$ is called \emph{square-integrable} if
\[
\int_D |f(x)|^2 \, dx < \infty.
\]
The space of such functions is denoted $L^2(D)$, a Hilbert space with inner
product $\langle f, g \rangle = \int_D f(x)g(x) \, dx$.

\section{Fourier Neural Operators}
\label{appendix:fno}
The Fourier Neural Operator (FNO), introduced by \citet{li2020fourier}, is a
deep learning architecture designed to learn mappings between
infinite-dimensional function spaces, especially solution operators of
parametric partial differential equations. Unlike traditional neural networks
that act on finite-dimensional vectors, FNOs learn operators of the form
\[
\mathcal{G}: f(x) \mapsto u(x), \quad f \in \mathcal{X},\ u \in \mathcal{Y},
\]
where $\mathcal{X}, \mathcal{Y}$ are typically subsets of $L^2(D)$ or $H^s(D)$
over a spatial domain $D \subset \mathbb{R}^d$. The central innovation of the
FNO is to parameterize the action of the operator in the \textbf{Fourier
domain}, allowing it to efficiently capture long-range dependencies and smooth
functional structure. FNO layers consist of a Fourier transform to move the
function into frequency space, a learned diagonal multiplier (analogous to a
convolution kernel) acting on each frequency mode, an inverse Fourier
transform to return to the spatial domain, and pointwise nonlinearities with
optional skip connections.

\subsection{Mathematical Structure of an FNO Layer}
Let $v: D \to \mathbb{R}^C$ be a function with $C$ channels. An FNO layer
updates $v$ as
\[
\text{FNO}(v)(x) = \mathcal{F}^{-1}\left( R(\hat{v}) \right)(x) + W(v(x)),
\]
where $\hat{v} = \mathcal{F}(v)$ is the Fourier transform, $R$ is a learned
transformation applied mode-wise (typically a complex-valued linear layer on
each frequency), and $W$ is a learned pointwise linear transformation. The
number of retained modes $k$ is typically truncated, introducing an implicit
low-pass filter that stabilizes training and improves generalization.

\end{document}